\documentclass{article}
\pdfoutput=1

\usepackage[preprint]{neurips_2022}

\usepackage{float}

\usepackage{wrapfig}

\usepackage{wrapfig,lipsum,booktabs}

\usepackage[utf8]{inputenc} 
\usepackage[T1]{fontenc}    
\usepackage{hyperref}       
\usepackage{url}            
\usepackage{booktabs}       
\usepackage{amsfonts}       
\usepackage{nicefrac}       
\usepackage{microtype}      
\usepackage{xcolor}         
\usepackage{enumitem}

\usepackage[ruled,vlined, linesnumbered]{algorithm2e}

\usepackage{subfigure}
\usepackage{graphicx}
\usepackage{algorithmic}
\usepackage{amsmath}
\usepackage[export]{adjustbox}
\usepackage{multirow}

\newcommand{\ourmethod}{FINGER}

\DeclareMathOperator*{\argmin}{arg\,min}

\title{FINGER: Fast Inference for Graph-based Approximate Nearest Neighbor Search}

\author{%
  Patrick H. Chen\\
  UCLA\\
  Los Angels, CA \\
  \texttt{patrickchen@g.ucla.edu} \\
   \And
   Wei-cheng, Chang \\
   Amazon \\
   Palo Alto, CA \\
   \texttt{weicheng.cmu@gmail.com} \\
    \And
   Hsian-fu, Yu \\
   Amazon \\
   Palo Alto, CA \\
   \texttt{rofu.yu@gmail.com} \\
   \And
   Inderjit S. Dhillon \\
   UT Austin $\&$ Amazon \\
   Palo Alto, CA \\
   \texttt{inderjit@cs.utexas.edu} \\
   \And
   Cho-jui, Hsieh \\
   UCLA $\&$ Amazon \\
  Los Angels, CA \\
   \texttt{chohsieh@cs.ucla.edu} \\
}

\usepackage{amsthm}

\begin{document}

\theoremstyle{plain}
\newtheorem{theorem}{Theorem}[section]
\newtheorem{proposition}[theorem]{Proposition}
\newtheorem{lemma}[theorem]{Lemma}
\newtheorem{corollary}[theorem]{Corollary}
\theoremstyle{definition}
\newtheorem{definition}[theorem]{Definition}
\newtheorem{assumption}[theorem]{Assumption}
\theoremstyle{remark}
\newtheorem{remark}[theorem]{Remark}

\maketitle

\begin{abstract}
Approximate K-Nearest Neighbor Search (AKNNS) has now become ubiquitous in modern applications, for example, as a fast search procedure with two tower deep learning models. Graph-based methods for AKNNS in particular have received great attention due to their superior performance. These methods rely on greedy graph search to traverse the data points as embedding vectors in a database. Under this greedy search scheme, we make a key observation: many distance computations do not influence search updates so these computations can be approximated without hurting performance. As a result, we propose \ourmethod{}, a fast inference method to achieve efficient graph search. \ourmethod{} approximates the distance function by estimating angles between neighboring residual vectors with low-rank bases and distribution matching. The approximated distance can be used to bypass unnecessary computations, which leads to faster searches. Empirically, accelerating a popular graph-based method named HNSW by \ourmethod{} is shown to outperform existing graph-based methods by 20$\%$-60$\%$ across different benchmark datasets. 
\end{abstract}

\section{Introduction}

$K$-Nearest Neighbor Search (KNNS) is a fundamental problem in machine learning \cite{bishop2006pattern}, and is used in various applications in computer vision, natural language processing and data mining \cite{chen2018learning,plotz2018neural,matsui2018survey}. Further, most of the neural embedding-based retrieval and recommendation algorithms require KNNS in the inference phase to find items that are nearest to a given query~\cite{zhang2019deep}. Formally, consider a dataset $D$ with $n$ data points $\{d_{1},d_{2},...,d_{n}\}$ where each data point has $m$-dimensional features. Given a query $q \in \mathbb{R}^{m}$, KNNS algorithms return the $K$ closest points in $D$ under a certain distance measure (e.g., $L2$ distance $\|\cdot\|_2$). Despite its simplicity, the cost of finding exact nearest neighbors is linear in the size of a dataset, which can be prohibitive for massive datasets in real time applications. It is almost impossible to obtain exact $K$-nearest neighbors without a linear scan of the whole dataset due to a well-known phenomenon called curse of dimensionality \cite{indyk1998approximate}. Thus, in practice, an exact KNNS becomes time-consuming or even infeasible for large-scale data. To overcome this problem, researchers resort to Approximate $K$-Nearest Neighbor Search (AKNNS). An AKNNS method proposes a set of $K$ candidate neighbors $T = \{t_{1},\cdots,t_{K}\}$ to approximate the exact answer. Performance of AKNNS is usually measured by recall@$K$ defined as $\frac{|T \cap A|}{K}$, where $A$ is the set of ground-truth $K$-nearest neighbors of $q$ in the dataset $D$. 
Most AKNNS methods try to minimize the search time by leveraging pre-computed data structures while maintaining high recall \cite{jayaram2019diskann}. There is a large body of AKNNS literature \cite{cai2019revisit,wang2020note,matsui2018survey}; most of the efficient AKNNS methods can be categorized into three categories: quantization methods, space partitioning methods and graph-based methods. In particular, graph-based methods receive extensive attention from researchers due to their competitive performance. Many papers have reported that graph-based methods are among the most competitive AKNNS methods on various benchmark datasets~\cite{cai2019revisit,wang2020note,aumuller2020ann,Fu2021-mj}.

 \begin{figure*}[t]
  \centering
\includegraphics[width=.9\linewidth]{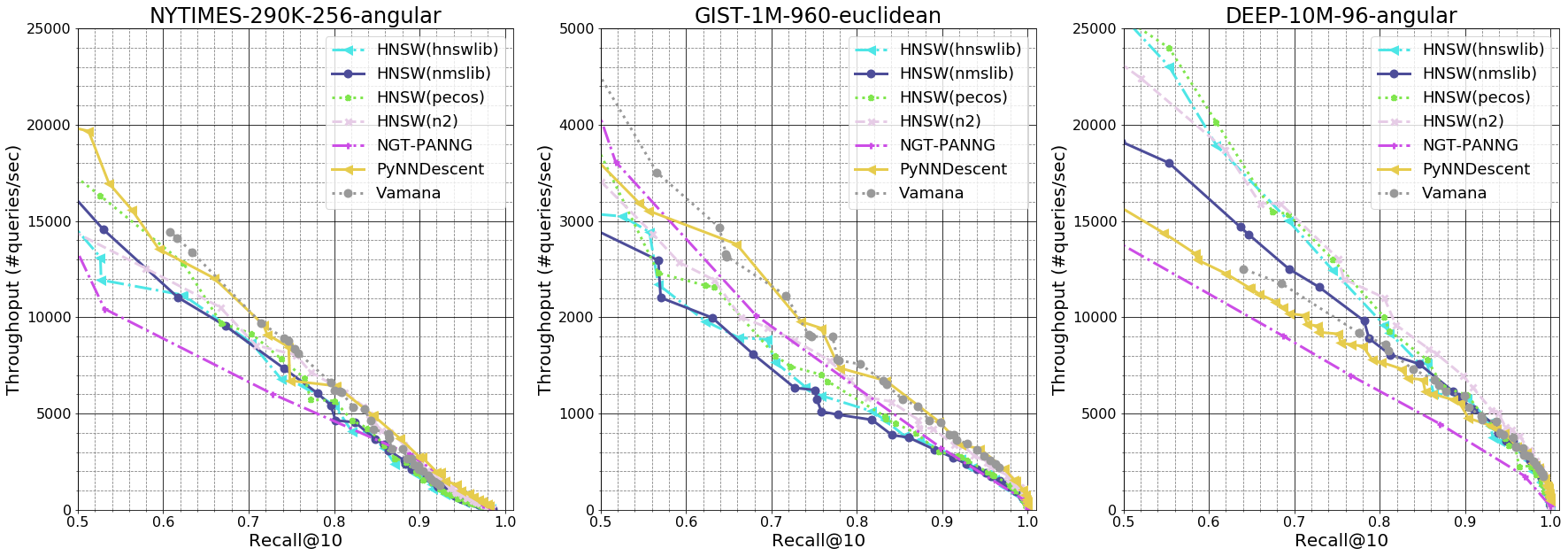}
  \vspace{-1em}
  \caption{ Comparison of state-of-the-art graph-based libraries on three benchmark datasets. Throughput versus recall@10 curve is used as the metric, where a larger area under the curve corresponds to a better method. We can observe no single method outperforms the rest on all datasets. }
  \label{fig::graph_result}
\end{figure*}

Graph-based methods work by constructing an underlying search graph where each node corresponds to a data point in $D$. Given a query $q$ and a current search node $c$, at each step, an algorithm will only calculate distances between $q$ and all neighboring nodes of $c$. Once the local search of $c$ is completed, the current search node will be replaced with an unexplored node whose distance is the closest to $q$ among all unexplored nodes. Thus, neighboring edge selection of a data point plays an important role in graph-based methods as it controls the complexity of the search space. Consequently, most recent research is focused on how to construct different search graphs or design heuristics to prune edges in a graph to achieve efficient searches \cite{jayaram2019diskann,Fu2021-mj,sugawara2016approximately,malkov2018efficient}. Despite different methods having their own advantages, there is no clear winner among these graph construction approaches on all datasets. Following a recent systematic evaluation protocol \cite{aumuller2020ann}, we evaluate performance by comparing throughput versus recall@10 curves, where a larger area under the curve corresponds to a better method. As shown in Figure \ref{fig::graph_result}, many graph-based methods achieve similar performance on three benchmark datasets. A method (e.g., PyNNDescent \cite{dong2011efficient}) can be competitive on a dataset (e.g., GIST-1M-960) while another method (e.g., HNSW \cite{malkov2018efficient}) performs better on the other dataset (e.g., DEEP-10M-96). These results suggest there might not be a single graph construction method that works best, which motivates us to consider the research question: \emph{Other than improving an underlying search graph, is there any other strategy to improve search efficiency of all graph-based methods?}.

 In this paper, instead of proposing yet another graph construction method, we show that for a given graph, part of the computations in the inference phase can be substantially reduced. 
 Specifically, we observe that after a few node updates, most of the  distance computations will not influence the search update. This suggests the complexity of distance calculation during an intermediate stage can be reduced without hurting performance. Based on this observation, we propose \ourmethod{}, \textbf{F}ast \textbf{IN}ference for \textbf{G}raph-based approximated nearest neighbor s\textbf{E}a\textbf{R}ch, which reduces computational cost in a graph search while maintaining high recall. Our contribution are summarized as follows:
\begin{itemize}[leftmargin=*]
    \item We provide an empirical observation that most of the distance computations in the prevalent best-first-search graph search scheme do not affect final search results. Thus, we can reduce the computational complexity of many distance functions.
    \item Leveraging this characteristic, we propose an approximated distance based on modeling angles between neighboring vectors using low-rank bases. In addition, angles of neighboring vectors in a graph tend to be distributed as a Gaussian distribution, and we propose a distribution matching scheme to achieve a better distance approximation. 
    \item We provide an open source efficient C++ implementation of the proposed algorithm \ourmethod{} on the popular HNSW graph-based method. HNSW-\ourmethod{} outperforms many popular graph-based AKNNS algorithms in wall-clock time across various benchmark datasets by 20$\%$-60$\%$. 
\end{itemize}

\vspace{-.3cm}
\section{Related Work}
\vspace{-.2cm}

There are three major directions in developing efficient approximate K-Nearest-Neighbours Search (AKNNS) methods. The first direction is still to traverse all elements in a database but reduce the complexity of each distance calculation; quantization methods represent this direction. The second direction is to partition search space into regions and only search data points falling into matched regions. This includes tree-based methods \cite{silpa2008optimised} and hashing-based methods \cite{charikar2002similarity}. The third direction is graph-based methods which construct a search graph and convert the search into a graph traversal. 


{\bf Quantization Methods} compress data points and represent them as short codes. Compressed representations consume less storage and thus achieve more efficient memory bandwidth usage \cite{guo2020accelerating}. In addition, the  complexity of distance computations can be reduced by computing approximate
distances with the pre-computed lookup tables. Quantization can be done by random projections \cite{li2019random}, or learned by exploiting structure in the data distribution \cite{marcheret2009optimal,morozov2019unsupervised}. In particular, the seminal Product Quantization method \cite{jegou2010product} separates data feature space into different parts and constructs a quantization codebook for each chunk. Product Quantization has become the cornerstone for most recent quantization methods \cite{guo2020accelerating,martinez2018lsq++,wu2017multiscale,Douze2018-kk}. There is also work focusing on learning transformations in accordance with product quantization \cite{ge2013optimized}. Most recent quantization methods achieve competitive results on various benchmarks \cite{guo2020accelerating,JDH17}.

{\bf Space Partition Methods} includes hashing-based and tree-based methods. Hashing-based Methods generate low-bit codes for high dimensional data and try to preserve the similarity among the original distance measure. Locality sensitive hashing~\cite{gionis1999similarity} is a representative framework  that enables users to design a set of hashing functions. Some data-dependent hashing functions have also been designed \cite{wang2010sequential,he2013k}. Nevertheless, a recent review \cite{cai2019revisit} reported the simplest random-projection hashing \cite{charikar2002similarity} actually achieves the best performance. According to this review, the advantage of hashing-based methods is simplicity and low memory usage; however, they are significantly outperformed by graph-based methods. Tree-based Methods learn a recursive space partition function as a tree following some criteria. When a new query comes, the learned partition tree is applied to the query and the distance computation is performed only on relevant elements falling in the same sub-tree. Representative methods are 
KD-tree \cite{silpa2008optimised} and $R^{*}$-tree \cite{beckmann1990r}. It is observed in previous studies that tree-based methods only work for low-dimensional data and their performances drop significantly for high-dimensional problems \cite{cai2019revisit}.

{\bf Graph-based Methods} date back to theoretical work in graph theory \cite{aurenhammer1991voronoi,lee1980two,dearholt1988monotonic}. However, these theoretical guarantees only work for low-dimensional data \cite{aurenhammer1991voronoi,lee1980two} or require expensive ($O(n^2)$ or higher) index building complexity \cite{dearholt1988monotonic}, which is not scalable to large-scale datasets. Recent works are mostly geared toward approximations of different proximity graph structures to improve nearest neighbor search. There is a series of works on approximating $K$-nearest-neighbour graphs \cite{harwood2016fanng,hajebi2011fast,fu2016efanna,jin2014fast}. Most recent works approximate monotonic graph \cite{Fu2017-tg} or relative neighbour graph \cite{arya1993approximate,malkov2018efficient}. In essence, these methods first construct an approximated $K$-nearest-neighbour graph and prune redundant edges by different criteria inspired by different proximity graph structures. Some other works mixed the above criteria with other heuristics to prune the graph \cite{Fu2021-mj,jayaram2019diskann}. Some pruning strategies can even work on randomly initialized dense graphs \cite{jayaram2019diskann}. According to various empirical studies \cite{harwood2016fanng,cai2019revisit,aumuller2020ann}, graph-based methods achieve very competitive performance among all AKNNS methods. Despite concerns 
about scalability of graph-based methods due to their larger memory usage \cite{Douze2018-kk}, it has been shown that graph-based methods can be deployed in billion scale commercial usage \cite{Fu2017-tg}. In addition, recent studies also demonstrated that graph-based AKNNS can scale quite well on billion-scale benchmarks when implemented on SSD hard-disks
%
\cite{jayaram2019diskann,chen2021spann}. In this work, we aim at demonstrating a generic method to accelerate the inference speed of graph-based methods so we will mainly focus on in-memory scenarios.


\vspace{-.3 cm}
\section{Methods}
\vspace{-.1cm}

\subsection{Observation: Most distance computations do not contribute to better search results}
\label{sec:observation1}
\vspace{-.1cm}

Once a search graph is built, graph-based methods use a greedy-search strategy (Algorithm \ref{alg:greedy}) to find relevant elements of a query in a database. It maintains two priority queues: candidate queue that stores potential candidates to expand and top results queue that stores current most similar candidates (line \ref{line:tc}). At each iteration, it finds the current nearest point in the candidate queue and explores its neighboring points. An upper-bound variable records the distance of the furthest element from the current top results queue to the query $q$ (line \ref{line:lb}). The search will stop when the current nearest distance from the candidate queue is larger than the upper-bound (line \ref{line:lbcriteria}), or there is no element left in the candidate queue (line \ref{line:emptycriteria}). The upper-bound not only controls termination of the search but also determines if a point will present in the candidate queue (line \ref{line:importantlb}). An exploring point will not be added into the candidate queue if the distance from the point to the query is larger than the upper-bound. Thus, upper-bound plays an important role as we need to spend computational resources on distance calculation  (dist function in line \ref{line:importantlb}) but it might not influence search results if the distance is larger than the upper-bound. Empirically, as shown in Figure \ref{fig::ratio}, we observe in two benchmark datasets that most of the  explorations end up having a larger distance than the upper-bound. 
Especially, starting from the mid-phase of a search, over 80 $\%$ of distance calculations are larger than the upper-bound. 
Using greedy graph search will inevitably waste  a significant amount of computing time on non-influential operations. \cite{li2020improving} also found this phenomenon and proposed to learn an early termination criterion by an ML model. Instead of only focusing on  the near-termination phase, we propose a more general framework by incorporating the idea of reducing the complexity of distance calculations into a graph search. The fact that most distance computations do not influence search results suggests that we don't need to have exact distance computations. A faster distance approximation can be applied in the search.

\begin{minipage}{.47\linewidth}

\begin{figure}[H]
  \centering
    \subfigure[]{
    \label{fig:fashionratio}
    \includegraphics[width=\linewidth]{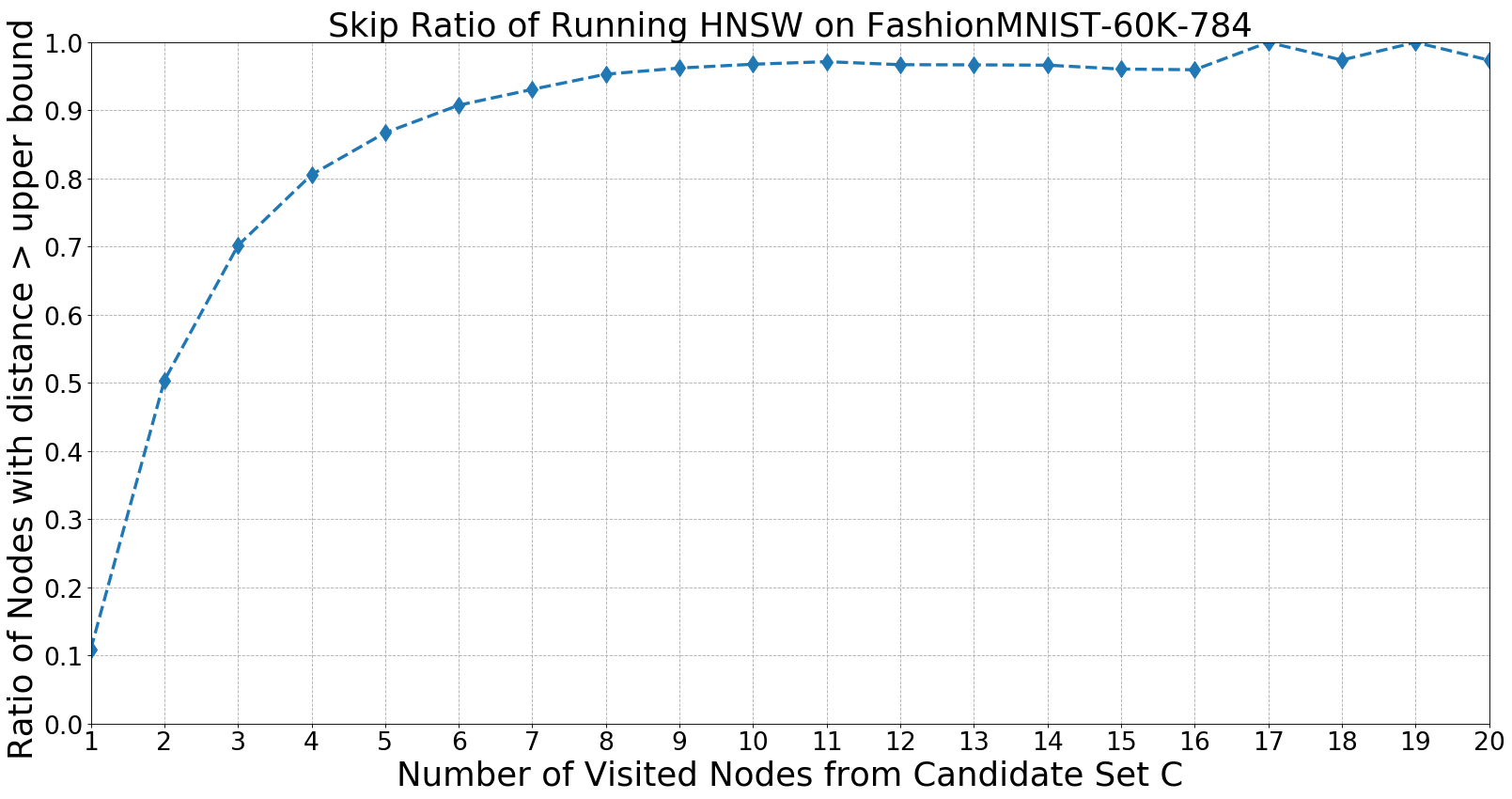}
    }
    \\
    \subfigure[]{
    \label{fig:gloveratio}
    \includegraphics[width=\linewidth]{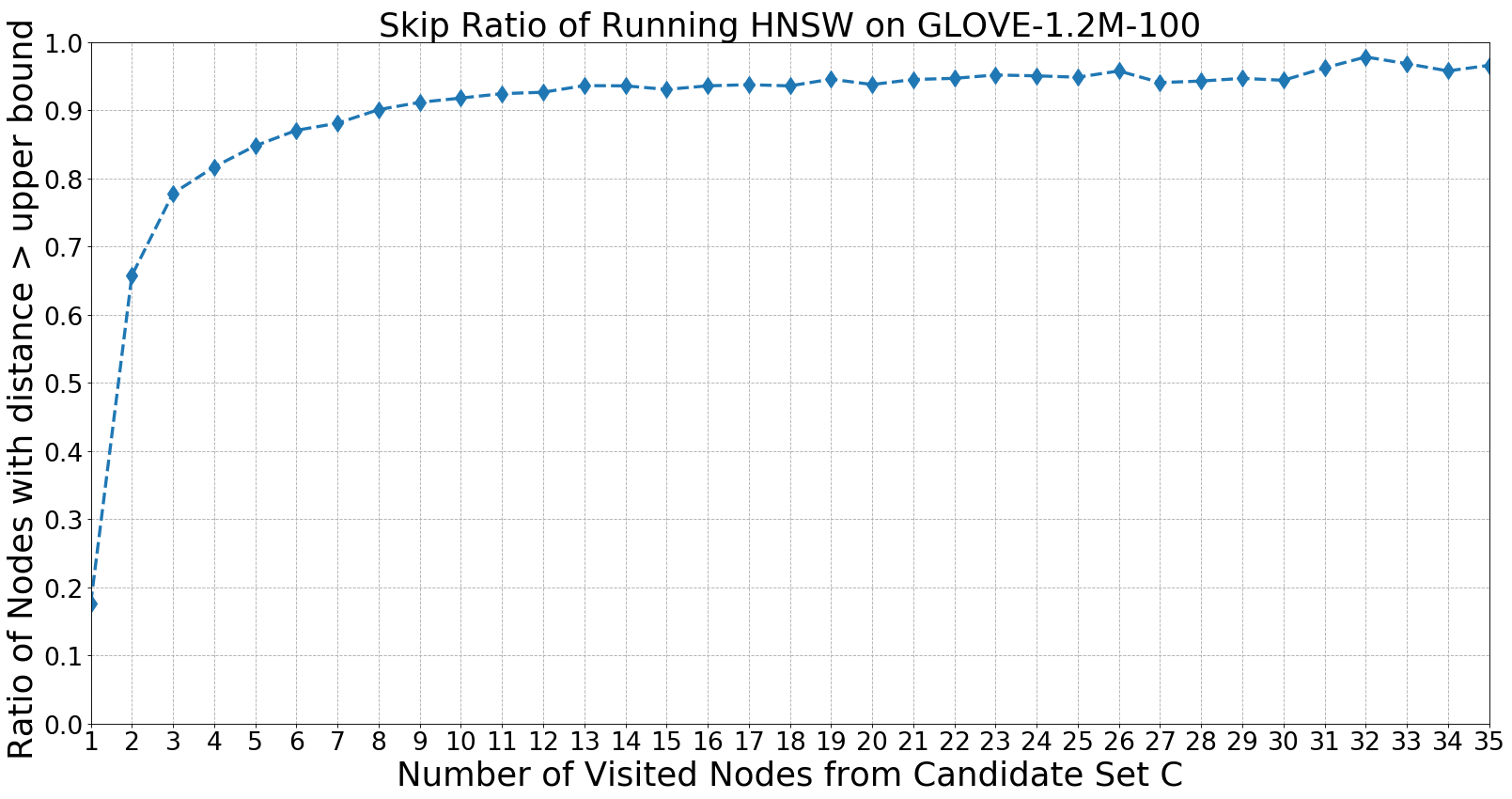}
    } 
  \vspace{-1em}
  \caption{ Illustration of the empirical observation that most points in a database will have distance to query larger than the upper-bound. (a)  FashionMNIST-60K-784 dataset (b) Glove-1.2M-100 dataset. Starting from the 5th step of greedy graph search (i.e., running line \ref{line:emptycriteria} in Algorithm \ref{alg:greedy} five times), both experiments show more than 80$\%$ of data points will be larger than the current upper-bound.    }
  \label{fig::ratio}\vspace*{-.3cm}
\end{figure}

\end{minipage}
\begin{minipage}{.03\linewidth}~\end{minipage}
\begin{minipage}{.5\linewidth}

\begin{algorithm}[H]
   \caption{Greedy Graph Search}
   \label{alg:greedy}
 \KwIn{graph $G$, query $q$, start point $p$, distance dist(), number of nearest points to return $efs$}
 \KwOut{top results queue $T$}
candidate queue $C = \{$p$\}$ \label{line:cs}, currently top results queue $T = \{$p$\}$ \label{line:tc}, visited $V = \{$p$\}$ \;

\While{$C$ is not empty \label{line:emptycriteria}}{ 
cur $\leftarrow$ nearest element from $C$ to $q$ (i.e., current nearest point to expand)  \label{line:cline} \;

 ub $\leftarrow$ distance of the furthest element from $T$ to $q$ (i.e., upper bound of the candidate search)  \label{line:lb}\;
 
    \If{ dist(cur, q) $>$ ub  \label{line:lbcriteria}} {return $T$ }

    \For{point $n$ $\in$ neighbour of cur in $G$  \label{inalg:expandline}}{
     
     \If{ $n$ $\in$ V}{ continue}
    
       V.add($n$)
       
      \If{dist(n, q) $\le$ ub or $|T|$ $\le$ $efs$ \label{line:importantlb}}
      {
       $C$.add(n) 
       
        $T$.add(n)
        
        \If{$|T|$ $>$ $efs$  } {remove furthest point to $q$ from $T$}
        
         ub $\leftarrow$ distance of the furthest element from $T$ to $q$ (i.e., update ub)
      }

    }
    
}

return $T$

\end{algorithm}
\vspace{-.2cm}
\end{minipage}

\subsection{Modeling Distribution of Neighboring Residual Angles}
\label{sec:observation2}


Given a query $q$ and the current nearest point to $q$ in the candidate queue $c$, 
in Line \ref{inalg:expandline} of Algorithm \ref{alg:greedy}, we will expand the search by exploring neighbors of $c$. 
Consider a specific neighbor of $c$ called $d$, we have to compute distance between $q$ and $d$ in order to update the search results. 
Here, we will focus on the $L2$ distance (i.e., $Dist = \|q - d\|_2$). The 
derivations of inner-product and angle distance are provided in the Supplementary A.
As shown in the previous section, most distance computations will not contribute to the search in later stages, we aim at finding a fast approximation of $L2$ distance. 
 A key idea is that we can leverage $c$ to represent $q$ 
(and $d$) as a vector along $c$ (i.e., projection) and a vector orthogonal to $c$ (i.e., residual):
\begin{align}
    q = q_{proj} + q_{res}, \quad 
    q_{proj} = \frac{c^Tq}{c^Tc}c, \quad
    q_{res} = q - q_{proj}. \label{eq:qres}
\end{align}


In other words, we treat each center node as a basis and project the query and its neighboring points onto the center vector so query and data can be written as $q = q_{proj} + q_{res}$ and $d = d_{proj} + d_{res}$  respectively. With this formulation, the squared $L2$ distance can be written as:
\begin{align}
    Dist^2 &= \|q - d\|_2^2 \nonumber
    = \|q_{proj} + q_{res}  - d_{proj} - d_{res}\|_2^2 \nonumber
    = \|(q_{proj} - d_{proj}) + (q_{res}   - d_{res})\|_2^2 \nonumber\\
    &=  \|(q_{proj} - d_{proj})\|_2^2 + \|(q_{res}  - d_{res})\|_2^2 \nonumber + 2(q_{proj} - d_{proj})^T(q_{res}  - d_{res})  \nonumber\\
    &\stackrel{(a)}= \|(q_{proj} - d_{proj})\|_2^2 + \|(q_{res}  - d_{res})\|_2^2 \nonumber\\
   &= \|(q_{proj} - d_{proj})\|_2^2 + \|q_{res}\|_2^2  + \|d_{res}\|_2^2 - 2 q_{res}^Td_{res} \label{eq:decomposition},
\end{align}
where (a) comes from the fact that projection vectors are orthogonal to residual vectors so the inner product vanishes. For $d_{proj}$ and $d_{res}$, we can pre-calculate these values after the search graph is constructed. For $q_{proj}$, notice that center node $c$ is extracted from the candidate queue (line \ref{line:cline} of Algorithm \ref{alg:greedy}). That means we must have already visited $c$ before. Thus, $\|q - c\|_2$ has been calculated and we can get $q^Tc$ by a simple algebraic manipulation:
\begin{align*}
    q^Tc = \frac{\|q\|_2^2 + \|c\|_2^2- \|q - c\|_2^2 }{2}.
\end{align*}


Since calculation of $\|q\|_2^2 $ is a one-time task for a query, it's not too costly when a dataset is moderately large. $\|c\|_2^2$ can again be pre-computed in advance so $q^Tc$ and thus $q_{proj}$ can be obtained in just a few arithmetic operations. Also notice that $\|q\|_2^2$ =  $\|q_{proj}\|_2^2$ + $\|q_{res}\|_2^2$ as $q_{proj}$ and $q_{res}$ are orthogonal, so we can get $\|q_{res}\|_2^2$ by calculating 
$\|q\|_2^2$ - $\|q_{proj}\|_2^2$ in few operations too.


Therefore, the only uncertain term in Eq.  (\ref{eq:decomposition}) is $q_{res}^Td_{res}$. If we can estimate this term with less computational resources, we can obtain a fast yet accurate approximation of $L2$ distance. Since we don't have direct access to the distribution of $q$ and thus $q_{res}$, we hypothesize we can instead use the distribution of residual vectors between neighbors of $c$ to approximate the distribution of $q_{res}^Td_{res}$ term. The rationale behind this is as we only approximate $q_{res}^Td_{res}$ when $q$ and $c$ are close enough (i.e., $c$ is selected in line \ref{line:cline} of Algorithm \ref{alg:greedy}), both $q$ and $d$ could be treated as near points in our search graph and thus interaction between $q_{res}$ and $d_{res}$ might be well approximated by ${{d^\prime}_{res}}^T d_{res}$, where ${d}^{\prime}$ is another neighbouring point of $c$ and ${d}^{\prime}_{res}$ is its residual vector. Empirically, as shown in the left column of Figure \ref{fig::hist}, angles between  residual vectors of sampled neighbors (i.e., $d,d^{\prime} \in \text{neighbor}(c)$) distributes like a Gaussian. In particular, compared to the distribution of direct inner-product ${d^{\prime}_{res}}^Td_{res}$ (right column of Figure \ref{fig::hist}), the distribution $\text{cos}(d^{\prime}_{res},d_{res})$ is less-skewed and thus more alike Gaussian. This motivates us to design an efficient approximator of $\text{cos}(q_{res},d_{res})$ and obtain $q_{res}^Td_{res}$ by $\|q_{res}\|_2 \|d_{res}\|_2\text{cos}(q_{res},d_{res})$.

\subsection{\ourmethod{}: Fast Inference by Low-rank Angle Estimation and Distribution Matching}




\begin{figure}

\begin{minipage}{.48\textwidth}
\vspace{15 pt}
    \hfill\includegraphics[width=1\linewidth]{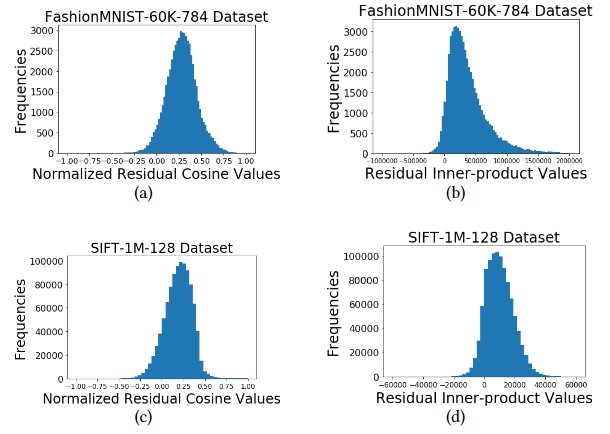}
    \vspace{-22pt}
    \caption{\small Illustration of the empirical observation that normalized cosine values of neighboring residual vectors distribute as a Gaussian distribution on FashionMNIST-60K-784 and SIFT-1M-128. Left column (a) and (c): angles of neighbouring residual pairs distribute alike Gaussian. Right column (b) and (d): un-normalized inner-product values between neighbouring residual pairs are more skewed. }
  \label{fig::hist}
\end{minipage}%
\hfill
\begin{minipage}{.48\textwidth}
    \hfill\includegraphics[width=.95\linewidth]{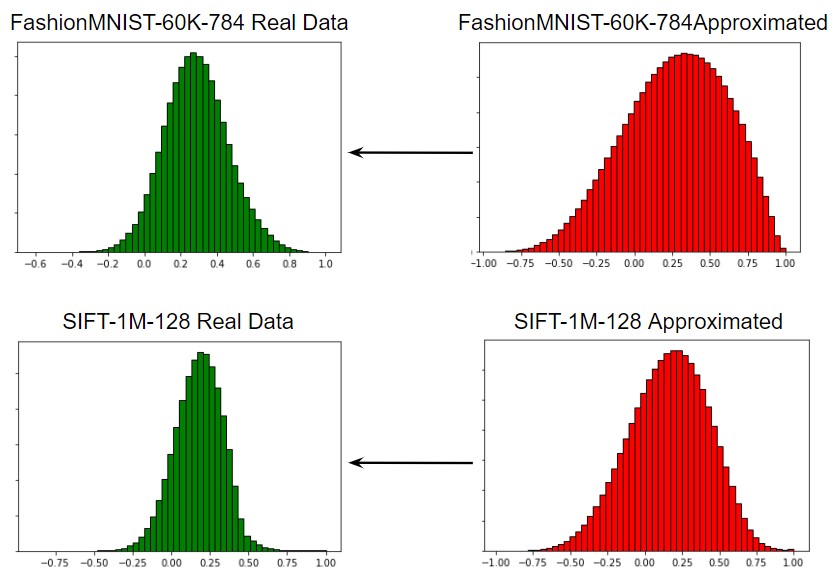}
    \caption{\small Illustration of Distribution Matching. In the left column, we show correct angle distributions of FashionMNIST-60K-784 and SIFT-1M-128. In the right column, we show angles of neighbouring residual pairs calculated by low-rank approximation ($r$ = 16). Our goal is to transform approximated results (in red) into real ones (in green).  }
    \label{fig:distmatch}
\end{minipage}

\end{figure}

\paragraph{\bf Low-rank Estimation}
Motivated by the above derivations, we aim at finding an efficient estimation of angles between all pairs of neighboring residual vectors. In AKNNS literature, a popular method for estimating this is Locality Sensitive Hashing (LSH) and its variants. In particular, Random Projection-based LSH (RPLSH) \cite{charikar2002similarity} is reported to achieve good average performance on various benchmark datasets \cite{cai2019revisit}. RPLSH samples $r$ random vectors from Normal distribution to form a projection matrix $P \in \mathbb{R}^{r \times m}$, where $m$ is the dimension of data and query.  
$L2$ distance between two vectors $x,y \in \mathbb{R}^{m}$ can be approximated by the distance in projected space, and the error,
\begin{align*}
    \left\|\| Px - Py \|_2^2 - \|x - y \|_2^2  \right\|_2,
\end{align*}
is bounded probabilistically \cite{johnson1984extensions}. We can further binarize the projection results to form a compact representation, and the angle between $x$ and $y$ can be approximated by hamming distance of the signed results: hamm(sgn($Px$),sgn($Py$)$)\frac{\pi}{r}$. However, there is an immediate disadvantage with this approach. Random projection guarantees worst case performance \cite{freksen2021introduction} and it is oblivious of the data distribution. 
Since we can sample abundant neighboring residual vectors from the training database, we can leverage the data information to obtain a better approximation. Formally, given an existing search graph $G=(D,E)$ where $D$ are nodes in the graph corresponding to data points and $E$ are edges connecting data points, we collect all residual vectors into $D_{res} \in \mathbb{R}^{m \times N}$, where $N$ is total number of edges in $G$ (i.e., $|E|$); and we assume $D_{res}$ spans the whole space which residual vectors lie in. The approximation problem can be formulated as the following optimization problem:
\begin{align}
\label{eq:optimization}
    \argmin_{P \in \mathbb{R}^{r \times m}}   \mathbb{E}_{x,y \sim D_{res}} \left\|\| Px - Py \|_2^2 - \|x - y \|_2^2  \right\|_2,
\end{align}
where we aim at finding an optimal $P$ minimizing the approximating error over the residual pairs $D_{res}$ from training data. It's not hard to see that the Singular Value Decomposition (SVD) of $D_{res}$ will provide an answer to the above optimization problem, and thus we can use SVD to find better $r$ lower-dimensions to estimate the angle of neighboring residual vectors. 
\begin{proposition}
\label{ref:proposition1}
Given a residual vector matrix $D_{res}  \in \mathbb{R}^{m \times N}$, and denoting $D_{res} = USV^T$ as the Singular Value Decomposition of $D_{res}$. $U_{1:r}$, the first $r$ columns of $U$ is an optimal solution of optimization problem Eq. (\ref{eq:optimization}).
\end{proposition}
\begin{proof}
The proof is provided in the Supplementary B.
\end{proof}

\paragraph{\bf Distribution Matching}
In addition to efficient low-rank estimation of angles, we further propose a distribution matching method to improve the performance. Despite as discussed  in Section \ref{sec:observation2} that angles between neighbouring residual vectors tend to be distributed alike Gaussian, this attribute only partially transfers to the distribution of angles approximated by low-rank computations as shown in Figure \ref{fig:distmatch}. Although the approximated distribution still looks alike Gaussian, its distribution is slightly skewed. Furthermore, its mean is shifted and its variance is larger than the real data distribution. To mitigate this, we propose to transform the approximated distributions into real data distributions by matching their mean and variance. Formally, assume angles of neighboring residual vectors follows a Gaussian distribution $\mathcal{N}(\mu,\sigma)$, and the approximated angles distributes as  $\mathcal{N}(\hat{\mu},\hat{\sigma})$.
Given a residual pair $x$ and $y$ with a low-rank projection matrix $P$, we can calculate the approximated angle $\hat{t} = \text{cos}(Px,Py)$. Under our assumption that it comes from a draw of $\mathcal{N}(\hat{\mu},\hat{\sigma})$, the value can be transformed by $ t = (\hat{t} - \hat{\mu})\frac{\sigma}{\hat{\sigma}} + \mu$.
 The transformed angle estimation $t$ then follows $\mathcal{N}(\mu,\sigma)$ as desired. Parameters  $\mu,\sigma,\hat{\mu},\hat{\sigma}$ can be estimated by using training data. \\
 

%
%


  \paragraph{\bf Overall Algorithm}
Construction of \ourmethod{} can be summarized in Algorithm \ref{alg:build}. Our aim is to provide a generic acceleration for all graph-based search. Thus, we can build the search index from any existing graph $G$. \ourmethod{} first iterates through all nodes in the graph. For each node, \ourmethod{} samples a pair of distinct nodes from its neighbors. In addition, we also calculate the residual vector of one sampled point and store it for later usage. We hypothesize the collected residual vectors $D_{res}$ spans the residual space, and we can find its optimal low-rank approximation by SVD. 
Once the low-rank projection $P$ is ready, we can estimate the mean and variance of angle distribution and approximated distribution respectively (i.e., line \ref{alg:mu_d}, \ref{alg:mu_a} in Algorithm \ref{alg:build}).
Certainly, this distribution matching scheme would still produce error. We further compute the average $L1$ error between real and approximated angles to serve as an error correction term. With this information saved in a search

\begin{figure*}
  \centering
    \includegraphics[width=.85\linewidth]{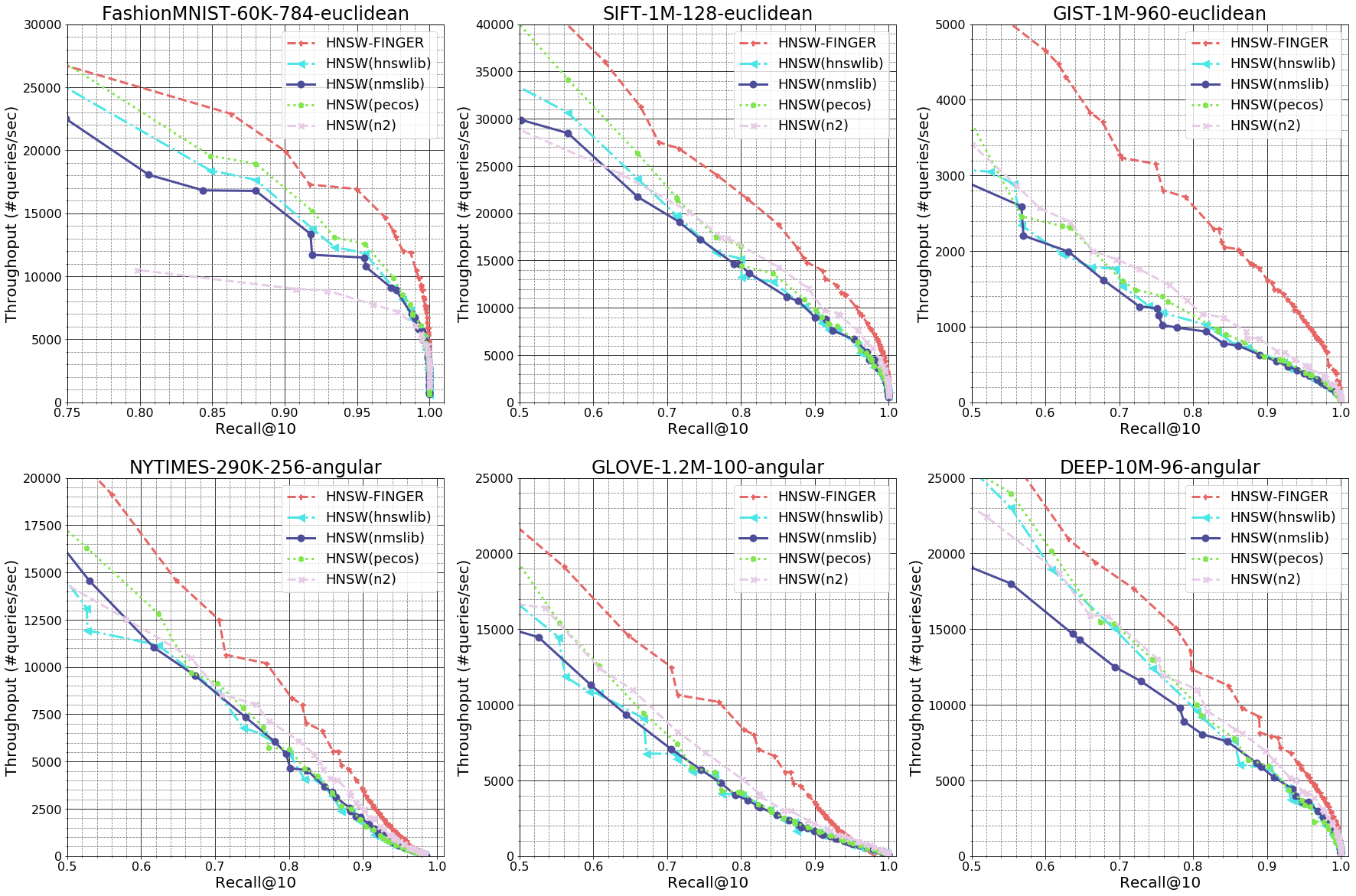}
  \vspace{-1em}
  \caption{ \small Experimental results of graph-based methods. Throughput versus Recall@10 chart is plotted for all datasets. Top row presents datasets with $L2$ distance measure and bottom row presents datasets with angular distance measure. We can observe a significant performance gain of \ourmethod{} over all existing graph-based methods. }
  \label{fig::main_result}
\end{figure*}

\begin{minipage}{.45\linewidth}
 index, Algorithm \ref{alg:appx_dist} approximates the distance between a query $q$ and a data point $d$. Notice that we explicitly write out the projection matrix $P$ and center node $c$ in Algorithm \ref{alg:appx_dist} to make it easier to understand the full approximation workflow. In practice, the projected residual vector $Pd_{res}$ can be pre-computed and stored. Detailed computation is illustrated in the Supplementary.

\section{Experimental Results}

\subsection{Experimental Setups}
\label{sec:experimental}

\paragraph{\bf Baseline Methods}
We compare \ourmethod{} to the most competitive graph-based and quantization methods. We include different  implementations of the popular HNSW methods such as NMSLIB \cite{malkov2018efficient}, n2\footnote{https://github.com/kakao/n2/tree/master}, PECOS \cite{yu2020pecos} and HNSWLIB \cite{malkov2018efficient}.
Other graph construction methods include NGT-PANNG \cite{sugawara2016approximately} , VAMANA(DiskANN) \cite{jayaram2019diskann} and PyNNDescent \cite{dong2011efficient}. Since our goal is to demonstrate \ourmethod{} can improve search efficiency of an underlying graph, we mainly include these competitive methods with good python interface and documentation. For quantization methods, we compare to the best performing ScaNN \cite{guo2020accelerating} and Faiss-IVFPQFS \cite{JDH17}.
In experiments, we combine \ourmethod{} with HNSW as it is a simple and prevalent method. The implementation of HNSW-\ourmethod{} is based on a modification of PECOS as its codebase is easy to read and extend. Pre-processing time and memory footprint are discussed in the Supplementary F.

\end{minipage}
\begin{minipage}{.03\linewidth}~\end{minipage}
\begin{minipage}{.53\linewidth}
  \begin{algorithm}[H]
\caption{Construction of \ourmethod{}}
 \label{alg:build}
     \setcounter{AlgoLine}{0}
 \KwIn{ graph $G = (D,E)$, rank $r$ }
 \KwOut{projection matrix $P$ and distribution parameters $\mu, \sigma, \hat{\mu}, \hat{\sigma}$, $\epsilon$ }
 
 $D_{res}$ = \{\}, $S$ = \{\}
 \For{$c \in D$} {
  Sample $d$,$d^{\prime} \in$  neighbors of $c$ and add it to $S$
  
  Calculate $d_{res}$,  $D_{res}$.add($d_{res}$)
   
  }
  
   Calculate SVD $U,S,V = \text{SVD}(D_{res})$
   
   $P$ = $U_{1:r}^T$, $X$ = \{\}, $Y$ = \{\}
   
   \For{pair $d,d^{\prime} \in S$ }{
    X.add(cos($d$,$d^{\prime}$)), Y.add(cos($Pd$,$Pd^{\prime}$))
   }
   
    $N = \text{size of} $ $X$,

   $\mu_{x}$ = $\frac{1}{N}$ $\sum\limits_{x \in X} x$,     $\sigma_{x} = $   $\frac{1}{N}$ $\sum\limits_{x \in X} (x - \mu_{x})^2$ \label{alg:mu_d}, 
   
   $\mu_{y}$ = $\frac{1}{N}$ $\sum\limits_{y \in Y} y$,     $\sigma_{y} = $   $\frac{1}{N}$ $\sum\limits_{y \in Y} (y - \mu_{y})^2$ \label{alg:mu_a}
   
   $\epsilon$ = $\frac{1}{N}$ $\sum\limits_{i=1}^{N}$  $\left|(Y_{i} - \mu_{y})\frac{\sigma_{x}}{\sigma_{y}} + \mu_{x} - X_{i} \right|$
   
   return $P$, $\mu, \sigma, \hat{\mu}, \hat{\sigma}$,$\epsilon$
   
\end{algorithm}

\begin{algorithm}[H]
    \setcounter{AlgoLine}{0}
   \caption{Approximate Distance Function}
   \label{alg:appx_dist}
   \KwIn{query $q$, projection matrix $P$, center node $c$, data point $d \in$ neighbors of $c$, distribution parameters $\mu, \sigma, \hat{\mu}, \hat{\sigma}$, $\epsilon$}
   \KwOut{ $t$, the approximated distance  between $q$ and $d$}

    compute $q_{res}$ and $d_{res}$ with $c$ and Eq. \ref{eq:qres}
    
    compute $\hat{t} = \text{cos}(Pq_{res},Pd_{res})$
    
    $t = (\hat{t} - \hat{\mu})\frac{\sigma}{\hat{\sigma}} + \mu$, t = t + $\epsilon$,   return t
\end{algorithm}
\end{minipage}

\vspace{-3pt}


\paragraph{\bf Evaluation Protocol and Dataset}
We follow ANN-benchmark protocol \cite{aumuller2020ann} to run all experiments. Instead of using a single set of hyperparameter, the protocol searches over a pre-defined set of hyper-parameters\footnote{https://github.com/erikbern/ann-benchmarks/blob/master/algos.yaml} for each method, and reports the best performance over each recall regime. In other words, it allows methods to compete others with its own best hyper-parmameters within each recall regime. We follow this protocol to measure recall@10 values and report the best performance over 10 runs. Results will be presented as throughput versus recall@10 charts. A method is better if the area under curve is larger in the plot. All experiments are run on AWS r5dn.24xlarge instance with Intel(R) Xeon(R) Platinum 8259CL CPU @ 2.50GHz.  We evaluate results over both $L2$-based and angular-based metric. We represent a dataset with the following format: (dataset name)-(training data size)-(dimensionality of dataset). For $L2$ distance measure, we evaluate on FashionMNIST-60K-784, SIFT-1M-128, and GIST-1M-960. For cosine distance measure, we evaluate on NYTIMES-290K-256, GLOVE-1.2M-100 and DEEP-10M-96. More details of each dataset can be found in \cite{aumuller2020ann}.

\begin{figure}
\centering
\includegraphics[width=1\linewidth]{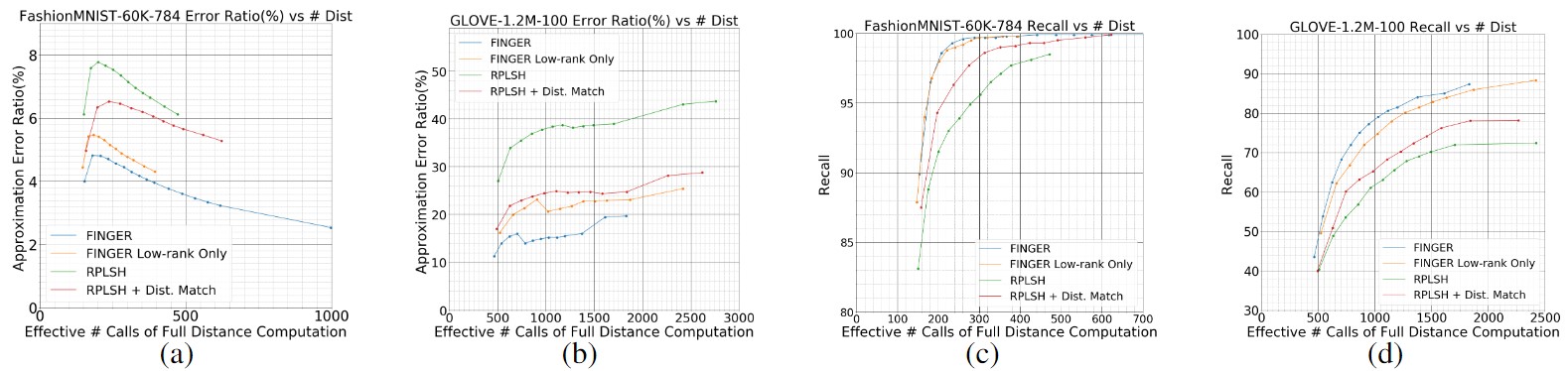}
\vspace*{-10pt}
  \caption{ \small Results of ablation studies on FashionMNIST-60K-784 and GLOVE-1.2M-100. (a) and (b) show approximation error($\%$) vs effective number of full distance calls. \ourmethod{} achieves smaller error than RPLSH. (c) and (d) show recall@10 vs effective number of full distance calls. \ourmethod{} achieves higher recalls.}
  \label{fig:aba2}
\end{figure}

\vspace{-.3cm}
\subsection{Improvements of \ourmethod{} over HNSW}
\vspace{-.3cm}
In Figure \ref{fig::main_result}, we demonstrate how \ourmethod{} accelerates the competitive HNSW algorithm on all datasets. Since \ourmethod{} is implemented on top of PECOS, it's important for us to check if PECOS provides any advantage over other HNSW libraries. Results verify that across all 6 datasets, the performance of PECOS does not give an edge over other HNSW implementations, so the performance difference between \ourmethod{} and other HNSW implementations could be mostly attributed to the proposed approximate distance search scheme. We observe that \ourmethod{} greatly boosts the performance over all different datasets and outperforms existing graph-based algorithms. \ourmethod{} works better not only on datasets with large dimensionality such as FashionMNIST-60K-784 and GIST-1M-960, but also works for dimensionality within range between 96 to 128. This shows that \ourmethod{} can accelerate the distance computation across different dimensionalities.
Results of comparison to most competitive graph-based methods are shown in Figure \ref{fig:total} of the Supplementary D. Briefly speaking, HNSW-\ourmethod{} outperforms most state-of-the-art graph-based methods except FashionMNIST-60K-784 where PyNNDescent achieves the best and HNSW-\ourmethod{} is the runner-up. Notice that \ourmethod{} could also be implemented over other graph structures including PyNNDescent. We chose to build on top of HNSW algorithm only due to its simplicity and popularity. Studying which graph-based method benefits most from \ourmethod{} is an interesting future direction. Here, we aim at empirically demonstrating approximated distance function can be integrated into the greedy search for graph-based methods to achieve a better performance.

\begin{figure*}
  \centering

\includegraphics[width=.85\linewidth]{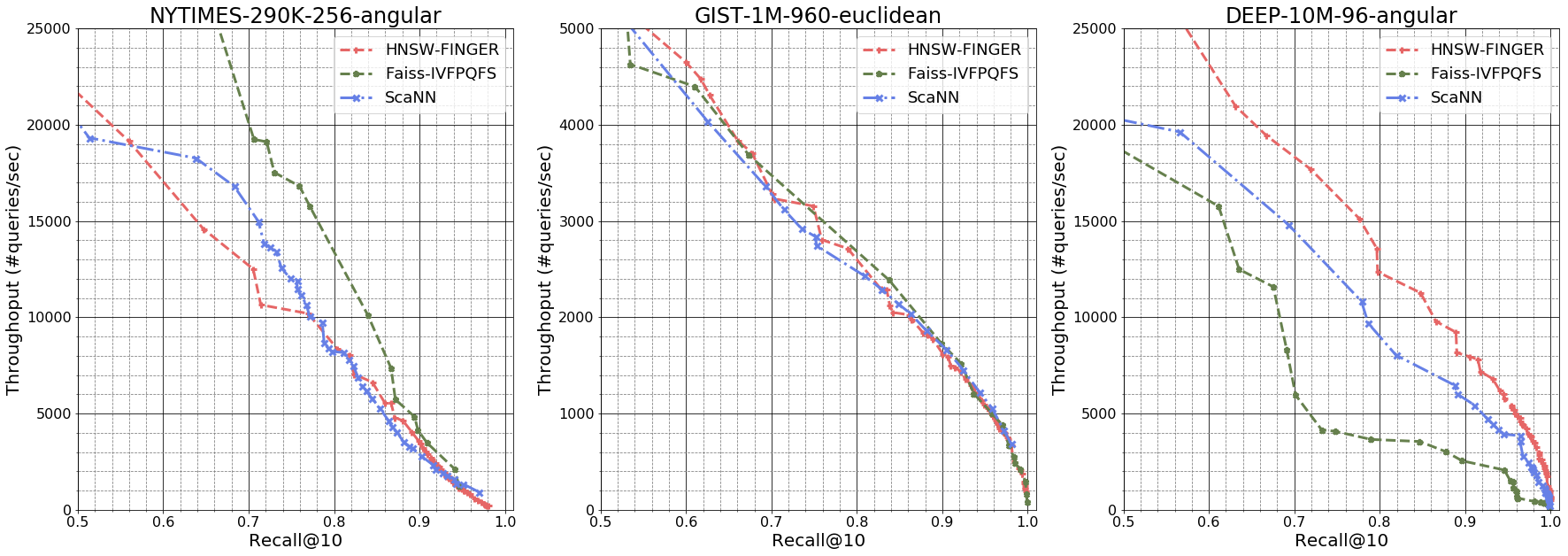}
  \caption{Comparisons to competitive quantization methods. Throughput versus Recall@10 chart is plotted for three datasets.
  We can observe each method has its pros and cons and there is no single method which performs best on all datasets.}
  \label{fig::quant_result}
  \vspace{-1em}
\end{figure*}

\vspace{-.2cm}
\subsection{Ablation Study}
\vspace{-.1cm}
We conduct an ablation study to see the effectiveness of each component of \ourmethod{}. First, we compare \ourmethod{} to the popular random projection locality hashing (RPLSH) for angle estimation. Since we have greatly optimized C++ implementation of \ourmethod{}, a direct comparison on wall-clock time won't be fair. Instead, we compare two schemes by counting the effective number of distance function calls. We collect the  number of full distance calls and approximate distance calls separately, and combine them into an effective number of distance calls. For example, if we call full $m$-dimensional distance $a$ times and $b$ times of $r$-dimensional approximate computations, we would have an effective distance calls of $a + b\frac{r}{m}$ times. We firstly analyze estimation quality by approximation error defined as $\frac{|t - \hat{t}|}{|t|}$ where $t$ is the true cosine angle value and $\hat{t}$ is the approximated value. Ideally, we could expect a better approximation scheme results in a smaller approximation error. Certainly, a smaller approximation doesn't necessarily yield better recall. Thus, we will also analyze the performance based on recall. Results of trade-off between approximation error ($\%$) and effective number of distance calls are shown in Figure \ref{fig:aba2}(a) for FashionMNIST-60K-784 and \ref{fig:aba2}(b) for GLOVE-1.2M-100. Corresponding results of recall vs effective distance calls are shown in Figure \ref{fig:aba2}(c) and \ref{fig:aba2}(d). We can see \ourmethod{} achieves smaller approximation errors compared to RPLSH on both datasets, which shows that \ourmethod{} is indeed a better low-rank approximation given data distribution. We also observe smaller approximation error transfers to higher recalls on both datasets. In addition, we apply distribution matching on RPLSH and found out this will greatly improve RPLSH. This shows distribution matching is a generic method that improves the performance of all different angle estimation methods. But even with the aid of distribution matching, RPLSH cannot achieve similar performance as \ourmethod{} and this shows the superiority of SVD results.
Since \ourmethod{} consists of a low-rank approximation module plus a distribution matching module, we are interested in studying their own effectiveness. We conduct similar analysis on full method and low-rank only version of \ourmethod{} shown in Figure \ref{fig:aba2}. Low-rank approximation alone still provides a much better angle estimation compared to RPLSH. Even without the distribution matching scheme, low-rank angle estimation outperforms RPLSH with distribution matching. We also observe limited difference between \ourmethod{} and \ourmethod{} without distribution matching in FashionMNIST-60K-784. However, the difference is more significant when it comes to GLOVE-1.2M-100 which still shows effective distribution matching. 

\vspace{-.4cm}
\subsection{Comparison to Quantization Results}
\vspace{-.3cm}
In addition to graph-based method, we are also interested in seeing the performance of HNSW-\ourmethod{} compared to the state-of-the-art quantization methods. Results of comparisons to quantization methods are shown in Figure \ref{fig::quant_result}. As we can observe, there is no single method achieving the best performance over all tasks.
Faiss-IVFPQFS performs well on NYTIMES-290K-256 but fails on DEEP-10M-96. ScaNN performs consistently well on all datasets but it doesn't achieve top performance on anyone. HNSW-\ourmethod{} performs competitively on GIST-1M-960 and DEEP-10M-96 but worse on NYTIMES-290K-256. These results showed that quantization provides some advantages over graph-based methods but the advantage is not consistent across datasets. Studying how to combine the advantage of quantization methods with \ourmethod{} and graph-based methods is an interesting future direction.

\vspace{-.4cm}
\section{Conclusions and Social Impact}
\vspace{-.4cm}
In this work, we propose \ourmethod{}, a fast inference method for graph-based AKNNS. \ourmethod{} approximates distance function in graph-based method by estimating angles between neighboring residual vectors. \ourmethod{} constructs low-rank bases to estimate residual angles and use distribution matching to achieve a better precision. The approximated distance can be used to bypass unnecessary distance evaluations, which translates into a faster searching. Empirically, \ourmethod{} on top of HNSW is shown to outperform all existing graph-based methods.

This work mainly focuses on accelerating existing models with approximate computations. It doesn't directly touch any controversial part of the data and thus it's unlikely providing any negative social impact. When used correctly with positive information to spread. It can help to accelerate the propagation as the work accelerates the inference speed.

\bibliographystyle{plain}
\bibliography{bibtex}

\section*{Checklist}

The checklist follows the references.  Please
read the checklist guidelines carefully for information on how to answer these
questions.  For each question, change the default \answerTODO{} to \answerYes{},
\answerNo{}, or \answerNA{}.  You are strongly encouraged to include a {\bf
justification to your answer}, either by referencing the appropriate section of
your paper or providing a brief inline description.  For example:
\begin{itemize}
  \item Did you include the license to the code and datasets? \answerYes{See Section~4.1.}
  \item Did you include the license to the code and datasets? \answerNo{The code and the data are proprietary.}
  \item Did you include the license to the code and datasets? \answerNA{}
\end{itemize}
Please do not modify the questions and only use the provided macros for your
answers.  Note that the Checklist section does not count towards the page
limit.  In your paper, please delete this instructions block and only keep the
Checklist section heading above along with the questions/answers below.

\begin{enumerate}

\item For all authors...
\begin{enumerate}
  \item Do the main claims made in the abstract and introduction accurately reflect the paper's contributions and scope?
    \answerYes{}
  \item Did you describe the limitations of your work?
    \answerYes{}
  \item Did you discuss any potential negative societal impacts of your work?
    \answerYes{}
  \item Have you read the ethics review guidelines and ensured that your paper conforms to them?
    \answerYes{}
\end{enumerate}

\item If you are including theoretical results...
\begin{enumerate}
  \item Did you state the full set of assumptions of all theoretical results?
    \answerNA{}
        \item Did you include complete proofs of all theoretical results?
    \answerNA{}
\end{enumerate}

\item If you ran experiments...
\begin{enumerate}
  \item Did you include the code, data, and instructions needed to reproduce the main experimental results (either in the supplemental material or as a URL)?
    \answerYes{}
  \item Did you specify all the training details (e.g., data splits, hyperparameters, how they were chosen)?
    \answerYes{}
        \item Did you report error bars (e.g., with respect to the random seed after running experiments multiple times)?
    \answerNo{}, the benchmark is using best result across different methods, so I followed the previous work.
        \item Did you include the total amount of compute and the type of resources used (e.g., type of GPUs, internal cluster, or cloud provider)?
    \answerYes{}
\end{enumerate}

\item If you are using existing assets (e.g., code, data, models) or curating/releasing new assets...
\begin{enumerate}
  \item If your work uses existing assets, did you cite the creators?
    \answerYes{}
  \item Did you mention the license of the assets?
    \answerNo{}, but it's included in the benchmark link I included.
  \item Did you include any new assets either in the supplemental material or as a URL?
    \answerNo{}
  \item Did you discuss whether and how consent was obtained from people whose data you're using/curating?
    \answerYes{}
  \item Did you discuss whether the data you are using/curating contains personally identifiable information or offensive content?
    \answerNo{}, This work is using existing dataset without touching these issues.
\end{enumerate}

\item If you used crowdsourcing or conducted research with human subjects...
\begin{enumerate}
  \item Did you include the full text of instructions given to participants and screenshots, if applicable?
    \answerNA{}{}
  \item Did you describe any potential participant risks, with links to Institutional Review Board (IRB) approvals, if applicable?
    \answerNA{}
  \item Did you include the estimated hourly wage paid to participants and the total amount spent on participant compensation?
    \answerNA{}
\end{enumerate}

\end{enumerate}


\clearpage
\appendix
\center{\Large Supplementary }
\section{Formulation of Inner-product}
In the main text, we presented derivation of $L2$ distance, and in this section we will derive the approximation for inner-product distance measure. Notice that angle measure can be obtained by firstly normalizing data vectors and then apply inner-product distance and thus the derivation is the same. For a query $q$ and data point $d$, inner-product distance measure is $Dist = q^Td$. Similar to $L2$ distance, we can apply the same decomposition to write $q = q_{proj} + q_{res}$ and $d = d_{proj} + d_{res}$. substituting the decomposition into distance definition, we have
\begin{align*}
    Dist = q_{proj}^Td_{proj} + q_{res}^Td_{res}.
\end{align*}
As in $L2$ case $ q_{proj}$ and $d_{proj}$ can be obtained by simple operations and the remaining uncertainy term is again $q_{res}^Td_{res}$. Therefore, in inner-product case, angle between neighboring residual vectors is still the target to approximate.

\section{Proof of proposition \ref{ref:proposition1}}

\begin{proof}
We can firstly construct all possible pairs of $\frac{N(N-1)}{2}$ combinations of sample of vectors $x,y$ from $D_{res}$ and compile all $\frac{N(N-1)}{2}$ pairs into two matrices $X$ and $Y$. With this notation, we can rewrite the original optimization into matrix form:
\begin{align*}
   &\argmin_{P \in \mathbb{R}^{r \times m}}   \mathbb{E}_{x,y \sim D_{res}}  \|\| Px - Py \|_2^2 - \|x - y \|_2^2  \|_2  \\
   &=      \argmin_{P \in \mathbb{R}^{r \times m}}  \| \| PX - PY \|_F^{2}  - \|X - Y\|_{F}^2 \|_2^2 \\
     &=      \argmin_{P \in \mathbb{R}^{r \times m}}  ( \| PX - PY \|_F^{2}  - \|X - Y\|_{F}^2 )^2,
\end{align*}
where $\|\cdot\|_F$ denotes matrix frobenius norm. By introducing the matrix notation, we then explicitly write out  the overall objective function without the sampling. We can further denote $Z = X - Y$. $Z$ matrix then denotes all possible pairs of vector difference from our original distribution. The objective function can then further be written into:
\begin{align*}
      &\argmin_{P \in \mathbb{R}^{r \times m}}  ( \| PX - PY \|_F^{2}  - \|X - Y\|_{F}^2 )^2 \\
      &= \argmin_{P \in \mathbb{R}^{r \times m}}  ( \| PZ \|_F^{2}  - \|Z\|_{F}^2 )^2 \\
      &= \argmin_{P \in \mathbb{R}^{r \times m}}  ( \| PU_{z}S_{z}V_{z}^T \|_F^{2}  - \|U_{z}S_{z}V_{z}^T\|_{F}^2 )^2,
\end{align*}
where $U_{z}S_{z}V_{z}^T$ denotes the SVD decomposition of $Z$. By the basic properties of SVD decomposition, we know that $\|U_{z}S_{z}V_{z}^T\|_{F}^2 = \| S_{z}\|_F^2$ as $U_{z}$ and $V_{z}$ are unitary matrices. $\| S_{z}\|_F^2$ equals sum of square of singular values of $Z$. Similarly,  $\|PU_{z}S_{z}V_{z}^T \|_F^{2} = \|PU_{z}S_{z} \|_F^{2}$. Thus it's not hard to see that the objective function is to find a projection direction which will result the minimal difference between the projected $S_{z}$ and full sum of squared eigenvalues of $S_{z}$. Thus, the optimal answer is the top $r$ directions as of columns of matrix $U_{z}$ as it will cancel out the top $r$ square of eigenvalues of $S_{z}$ which happens to be the largest ones.

The remaining thing is to show that SVD of $Z$ is essentially the same as SVD of $D_{res}$. Notice that both $X$ and $Y$ are just duplicating and re-ordering of $D_{res}$. So both $X$,$Y$ share the same basis of $D_{res}$. Denote SVD results of $D_{res} = USV^T$. We can then represent $X = USV_{x}^T$ and $Y = USV_{y}^T$. Consequently, we can also represent the SVD of $Z$ as  $Z = X - Y = USV_{x}^T - USV_{y}^T = US(V_{x}- V_{y})^T$ so we can see that it shares the same basis as $D_{res}$ and the proof is complete.

\end{proof}

\section{Approximate Greedy Search Algorithm}

\begin{algorithm}[H]
   \caption{Approximate Greedy Graph Search}
   \label{alg:appx_greedy}
   \KwIn{graph $G$, query $q$, starting point $p$, distance function dist(), appxoaimate distance function appx(), number of nearest points to return $efs$}
   \KwOut{top candidate set $T$}
   candidate set $C = \{$p$\}$ 
   
   dynamic list of currently best candidates $T = \{$p$\}$ 
   
   visited $V = \{$p$\}$

\While{$C$ is not empty}{ 
 cur $\leftarrow$ nearest element from $C$ to

E ub $\leftarrow$ distance of the furthest element from $T$ to $q$ (i.e., upper bound of the candidate search)

    \If{dist(cur, q) $>$ ub}{return T}
    
    \For{point $n$ $\in$ neighbour of cur in $G$}{
       \If{$n$ $\in$ V}{continue}
       
       V.add($n$)
       
       \uIf{ $\#$updates of cur > 5 times}{ e = appx(n, q)}\Else{ e = dist(n, q)}
       
       \If{ e $\le$ ub or $|T|$ $\le$ $efs$} {
       
       update distance to be dist(n,q)
       
        C.add(n)
        
        T.add(n)
        
        \If{$|T|$ $>$ $efs$}{remove furthest point to $q$ from T}
       
       ub $\leftarrow$ distance of the furthest element from $T$ to $q$ (i.e., update ub)
       }
    }
 }
 
 return T

\end{algorithm}

\begin{figure*}
  \centering
    \includegraphics[width=.9\linewidth]{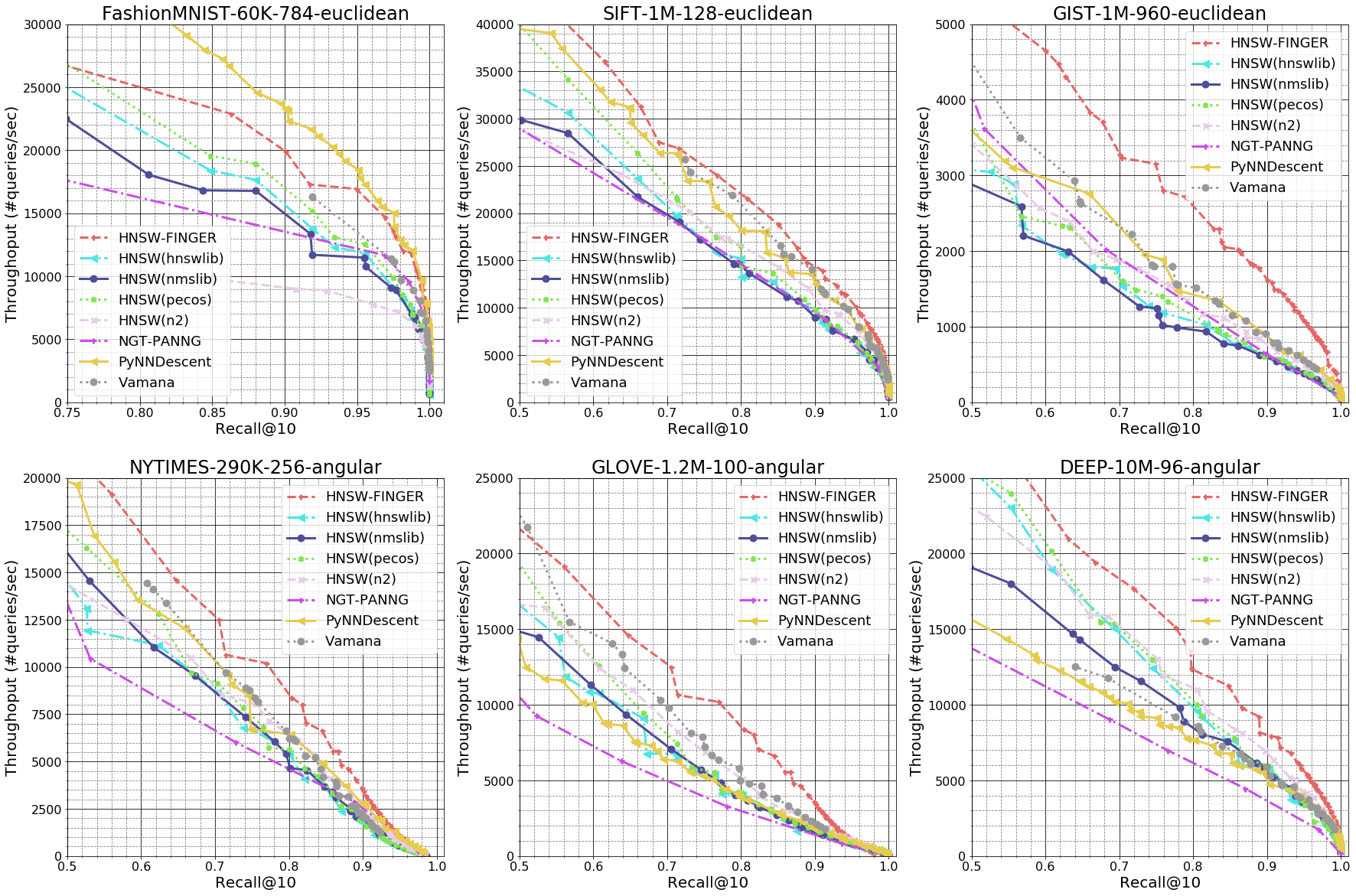}
  
  \caption{ Experimental results of graph-based methods. Throughput versus Recall@10 chart is plotted for all datasets. Top row presents datasets with $L2$ distance measure and bottom row presents datasets with angular distance measure. We can observe a significant performance gain of HNSW-\ourmethod{} over existing graph-based methods. }
  \label{fig:total}
\end{figure*}

\section{Complete Comparison of Graph-based Methods}
Complete results of all graph-based methods are shown in Figure \ref{fig:total}. HNSW-\ourmethod{} basically outperforms all existing graph-based methods except on FashionMNIST-60K-784 where PyNNDescent performs extremely well. In principle, \ourmethod{} could also be applied on PyNNDescent to further improve the result. Results show that currently no graph-based methods completely exploits the training data distribution. This reflects the importance of the inference acceleration methods as \ourmethod{} that can create consistently faster inference on all underlying search graph. Making a search graph maximally suitable for applying \ourmethod{} is also an interesting future direction.

\section{Selection of Rank Parameter $r$}

Under ANN-benchmark protocol, we could have made the selection of rank $r$ in \ourmethod{} as a hyper-parameter to search in order to achieve best performance. But this might be time-consuming for real applications. Instead, here we provide a practical rule of thumb for choosing $r$ by calculating the correlation coefficient of $X,Y$ in Algorithm \ref{alg:build}. $X$ stores true angles between neighboring pairs and $Y$ stores approximated angles. We start $r$ to be 8 in order to maximally leverage SIMD. Specifically, AVX2 SIMD allows a single instruction with 8 parallel floating point computation. Increase the rank in a multiple of 8 will maximally leverage the capability of SIMD instructions. Now, if the correlation is smaller than 0.7, we enlarge $r$ by 8 and redo Algorithm \ref{alg:build} again with increased $r$ until correlation between $X$ and $y$ is larger than 0.7. In this work, to show the effectiveness of applying \ourmethod{} in read world applications, we use this search scheme and ranks learned in FashionMNIST-60K-784: 16, SIFT-1M-128: 16, GIST-1M-960: 16, NYTIMES-290K-256: 48, GLOVE-1.2M-100: 32 and DEEP-10M-96: 24.

\section{Pre-processing Time and Memory footprint of HNSW-\ourmethod{} and HNSW}
Examples of pre-processing time and memory footprint of HNSW-\ourmethod{} and HNSW is shown in Table \ref{tab:memory}. \ourmethod{} requires additional linear scan of training data, so it will add some additional processing time to the base method. The difference is around 90 seconds which is not significant compared to the pre-processing time of base HNSW method. Memory usage of HNSW is approximately memory of data plus number of edges $|E|$ $\times$ sizeof(int). For a selected low-rank dimension $r$, \ourmethod{} requires additional ($r$ + 2) $\times$ |E| $\times$ sizeof(float) to store the pre-computed values.
\begin{table}
  \centering
   \vspace{-1mm}
    \caption{Construction statistics of HNSW-\ourmethod{} and HNSW. Pre-processing time in second is shown in the table. Numbers in parentheses represent the memory footprint in GB.}
\label{tab:memory}
  
\resizebox{7cm}{!}{
  \begin{tabular}{|c|c|c|c|}
    \hline
    Dataset & M & HNSW-\ourmethod{} & HNSW(PECOS) \\
    \hline 
    \multirow{2}{*}{SIFT-1M-128} & 12 &   291.5s (2.8G) &  215.4s (1G) \\
  \cline{2-4}
    & 48 & 521.4s (9.2G)  &  433.9s (2.4G) \\
    \hline
  \multirow{2}{*}{ GLOVE-1.2M-100}& 12 &   386.2s(4.8G) & 300.1s (1.1G) \\
   \cline{2-4}
    &  48 & 1409.8s (18G) & 1317.3s (2.7G) \\
    \hline
  \end{tabular}
  }
\end{table}

\section{Detailed computation of Approximate Distance.}
As mentioned in the main text, we explicitly write out the projection matrix $P$ and center node $c$ in Algorithm \ref{alg:appx_dist} to make it easier to understand the full approximation workflow. In practice, the projected residual vector $Pd_{res}$ can be pre-computed and stored in the search index to save inference time. $Pq_{res}$ can also be easily calculated by \begin{align*}
    Pq_{res} &= P(q - q_{proj}) 
    = Pq - Pq_{proj} 
    = Pq - \frac{c^Tq}{c^Tc}Pc.
\end{align*}
$Pq$ needs only be calculated once for whole graph search so the cost is limited when the rank $r$ is not large. As we point out in Section \ref{sec:observation1}, $c^Tq$ must already be calculated when we explore neighbours of $c$. $c^Tc$ and $Pc$ can again be pre-computed and stored. Thus, the cost of $Pq_{res}$ is limited to a low-dimensional vector subtraction. To apply \ourmethod{}, we only need to slightly modify  Algorithm \ref{alg:greedy}. Specifically, we replace distance function dist($n$,$q$) in line \ref{line:lbcriteria} of Algorithm \ref{alg:greedy} with the approximation (i.e., Algorithm \ref{alg:appx_dist}). If the approximated distance is larger than the upper-bound variable, the search continues. Otherwise, we calculate precise distance between $q$ and $d$ and update the candidate queue correspondingly. This will make sure that all distance information in candidates set $C$ is correct so the algorithm won't terminate too early. The complete modified algorithm is listed in the Supplementary C, and the selection of rank $r$ is discussed in Supplementary E.

\end{document}